\PassOptionsToPackage{hidelinks}{hyperref}
\documentclass[journal]{IEEEtran}

\usepackage[utf8]{inputenc}
\usepackage[T1]{fontenc}
\usepackage{lmodern}

\usepackage{mathtools,amssymb,amsthm}
\usepackage{algpseudocode,algorithm,algorithmicx}
\usepackage{array}
\usepackage{graphicx}
\usepackage{xcolor}
\usepackage{caption,subcaption}
\usepackage[export]{adjustbox}
\usepackage{tabularx}
\usepackage{multirow}
\usepackage{diagbox}
\usepackage{enumitem}
\usepackage{cite,bookmark}
\usepackage{hyphenat}
\usepackage{easyReview,todonotes}

\theoremstyle{plain}
\newtheorem{theorem}{Theorem}
\newtheorem{lemma}[theorem]{Lemma}

\newtheorem{corollary}[theorem]{Corollary}

\theoremstyle{definition}
\newtheorem{definition}{Definition}
\newtheorem{example}{Example}

\theoremstyle{remark}

\newcolumntype{C}[1]{>{\centering\arraybackslash}p{#1}}

\begin{document}

\title{Ensemble Recognition in Reproducing Kernel Hilbert Spaces through Aggregated Measurements}

\author{Wei~Miao,~
        Gong~Cheng,~
        and~Jr-Shin~Li
\thanks{The authors are with the Department
of Electrical and Systems Engineering, Washington University in St. Louis, Saint Louis,
MO, 63130 USA e-mail: {\texttt \{weimiao, gong.cheng, jsli\}@wustl.edu}}}


\maketitle
\begin{abstract}
  In this paper, we study the problem of learning dynamical properties of ensemble systems from their collective behaviors using statistical approaches in reproducing kernel Hilbert space (RKHS). Specifically, we provide a framework to identify and cluster multiple ensemble systems through computing the maximum mean discrepancy (MMD) between their aggregated measurements in an RKHS, without any prior knowledge of the system dynamics of ensembles. Then, leveraging the gradient flow of the newly proposed notion of aggregated Markov parameters, we present a systematic framework to recognize and identify an ensemble systems using their linear approximations. Finally, we demonstrate that the proposed approaches can be extended to cluster multiple unknown ensembles in RKHS using their aggregated measurements. Numerical experiments show that our approach is reliable and robust to ensembles with different types of system dynamics.
\end{abstract}


\IEEEpeerreviewmaketitle

\section{Introduction}
The study of ensemble systems, which are collections of dynamical systems parameterized by a continuous variable, has drawn much attention due to its broad applicability in diverse scientific areas. Notable examples include exciting a population of nuclear spins on the order of Avogadro's number in nuclear magnetic resonance (NMR) spectroscopy and imaging \cite{glaser1998unitary,li2011optimal}, spiking population of neurons to alleviate brain disorders such as Parkinson's disease \cite{brown2004phase,li2013control,kafashan2015optimal}, manipulating a group of robots under model perturbation \cite{becker2012approximate}, creating synchronization patterns in a network of coupled oscillators \cite{rosenblum2004controlling,Li_NatureComm16}, and explaining the functionality of neural networks, especially ultra-deep neural networks \cite{weinan2017proposal, tabuada2020universal}.

During the past decade, a rich amount of work has been developed focusing on system-theoretic analysis of ensemble systems. It has been shown that various specialized techniques such as polynomial approximation \cite{li2011ensemble}, separating points \cite{li2019separating}, representation theory \cite{chen2019structure},complex functional analysis \cite{helmke2014uniform,schonlein2016controllability,dirr2018uniform}, statistical moment-based approaches \cite{zeng2016moment,zeng2017sampled, kuritz2018ensemble}, and convex-geometric approaches \cite{miao2020convexgeometric} are nontrivially connected to analyzing ensemble controllability and ensemble observability. Apart from investigating fundamental properties of ensemble systems, customized approaches are proposed to design feasible and optimal control laws for linear \cite{Li_ACC12_SVD, miao2020geometric,tie2017explicit, zeng2018computation,miao2020numerical}, bilinear \cite{Li_SICON17,Li_Automatica18}, as well as specific types of nonlinear ensemble systems \cite{li2013control}.

One unique characteristic of ensemble systems is that the control and measurements can be collected only from a population level. This is because there are too many systems in the ensemble so that precisely tracking and applying feedback to each system in the ensemble is not feasible. As a consequence, the control signal is broadcast to all systems in the ensemble, and the measurements are aggregated across systems which represent collective behaviors of the ensemble. In many emerging scientific areas involving ensemble systems, it is of great interest to learn dynamical properties of ensemble systems in a data-driven manner, especially through collective behaviors characterized by so-called `aggregated measurements'. However, due to the unexplored mathematical structures, classical control-theoretic tool needs to be upgraded to address new challenges stemming from ensemble systems with aggregated measurements.

Specifically, by perturbing ensemble systems with random control signals, we can compare whether two ensemble systems possess similar collective behavior through computing the maximum mean discrepancy (MMD) between their aggregated measurements, without any prior knowledge of the system dynamics of ensembles. Then, leveraging on a gradient flow of the newly proposed notion of aggregated Markov parameters, we present a systematic framework to recognize and identify an ensemble systems using their linear approximations. Finally, we demonstrate that the proposed approaches can be extended to cluster multiple unknown ensembles in RKHS using their aggregated measurements.

In this paper, we study the problem of learning dynamical properties of ensemble systems from their collective behaviors using statistical approaches in reproducing kernel Hilbert space (RKHS). This paper is organized as follows. In Section \ref{sec: prelim}, we provide preliminary knowledge on ensemble systems, their aggregated measurements, reproducing kernel Hilbert space and two-sample test in RKHS. In Section \ref{sec: method}, we introduce the main result of this paper. Specifically, by perturbing ensemble systems with random control signals, we compare whether two ensemble systems possess similar collective behavior through computing the maximum mean discrepancy (MMD) between their aggregated measurements, without any prior knowledge of the system dynamics of ensembles. Then, leveraging on a gradient flow of the newly proposed notion of aggregated Markov parameters, we present a systematic framework to recognize and identify an ensemble systems using their linear approximations. Finally, we demonstrate that the proposed approaches can be extended to cluster multiple unknown ensembles in RKHS using their aggregated measurements. In Section \ref{sec: experiments}, we provide various examples to show the efficacy of our proposed ensemble recognition and clustering approach.

\section{Preliminaries} \label{sec: prelim}

\subsection{Ensemble Systems and Their Aggregated Measurements}

Consider an ensemble of dynamical systems indexed by $\beta$ over a compact set $K \subset \mathbb{R}$, given by
\begin{equation}
  \label{equ: ensemble definition}
  \begin{aligned}
    \Sigma: \frac{\mathrm{d}}{\mathrm{d}t} X(t, \beta) &= F(t, \beta, X(t, \beta), u(t)), \\
    X(0, \beta) &= X_0(\beta),
  \end{aligned}
\end{equation}
where $X(t, \cdot )\in L^2(K, \mathbb{R}^n)$ is the state, an $n$-tuple of $L^2$-functions over $K$ for each $t\in [0, T]$ with $T\in (0, \infty)$; $F$ and $G$ are $\mathbb{R}^n$-valued and $\mathbb{R}^q$-valued smooth functions, respectively; and $u\in \mathcal{U}$ is the $\mathbb{R}^m$-valued control signal, where $\mathcal{U}$ denotes the admissible control set; $Y(t, \cdot)$ is the observation of the ensemble. Such a population of dynamical systems defined on the space of functions as in \eqref{equ: ensemble definition}, is called an ensemble system (or an ensemble as an abbreviation). The unique characteristic of an ensemble system is that the control can only be implemented on a population level. Namely, the control signal $u(t)$ is broadcast to all the systems in the ensemble and therefore is independent from $\beta$. This is mainly because the number of systems in the ensemble can be exceedingly large so that feedback control on each system is no longer applicable.

Besides the under-actuated nature arising from the broadcast control signals, many ensemble systems also suffer from inadequate measuring techniques so that precisely tracking each system is not possible. In these cases, the index of each system, i.e., $\beta$, is not available for observation. As a result, the states of the ensemble, i.e., $X(t, \beta)$, can observed only through the so called ``aggregated measurements'' taking form of
\begin{equation}
  \label{equ: aggregated measurements definition}
  Y(t) = \int_K G(t, \beta, X(t, \beta), u(t)) \,\mathrm{d}\beta,
\end{equation}
where $G$ is an $\mathbb{R}^q$-valued smooth function, called the ``aggregated observation function''. To fix ideas, we provide three examples of aggregated measurements in practical applications involving ensemble systems.

\begin{example}[Measuring temperature]
  Let us consider a steel plate that is being uniformly heated from the bottom. In steel-making procedures, it is essential to estimate the temperature distribution of the plate, denoted as $\rho(t, x)$ in real-time using a thermometer, where $x$ denotes the position of an infinitesimal element of the steel plate. In this case, the temperature distribution $\rho(t, x)$ can be considered as the state variable of an ensemble system indexed by $x$. However, the measurement collected from a thermometer placed at $x=x_0$, given by
  \[
    y(t) = \frac{1}{2\epsilon} \int_{x_0-\epsilon}^{x_0+\epsilon} \rho(\sigma, t)\,\mathrm{d}\sigma
  \]
  is an aggregated measurement of the temperature distribution since the probe of thermometer (drawn as black object in Figure~\ref{fig: aggregated measurement demo temperature}) has a finite width of $2\epsilon$.
  \begin{figure}[!h]
    \centering
    \includegraphics[width=0.8\linewidth]{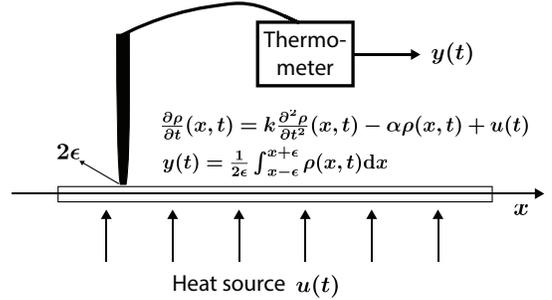}
    \caption{A demonstration of aggregated measurements by measuring temperature distribution of a steel plate using a thermometer.}
    \label{fig: aggregated measurement demo temperature}
  \end{figure}
\end{example}

\begin{example}[Torque by a hinged rod]
  Let us consider a rod of length $L$ that is horizontally placed with one end hinged on a vertical wall, as shown in Figure~\ref{fig: aggregated measurement demo}. We index each infinitesimally small element of the rod using its horizontal distance to the wall, denoted as $\beta$. Then, the linear densities $\lambda(t, \beta)$ of the rod can be considered as the state variable of an ensemble system indexed by $\beta$. Estimating $\lambda(t, \beta)$ is critical in many applications of engineering mechanics. However, in practice, it is impossible to measure the linear density $\lambda$ due to the inability of probing infinitesimal elements of the rod. Instead, realizable measurements are ``aggregated'' over all elements on the rod. For instance, we can measure the torque applied on the hinge by all elements on the rod, denoted as $\tau(t)$, through setting a torque wrench on the hinge. In this case, $\tau(t)$ is given by
  \[
    \tau(t) = \int_{0}^{L} \beta\lambda(t, \beta)g\, \mathrm{d}\beta,
  \]
  where $g$ is the acceleration due to the gravity. In this case, $\tau(t)$ can be considered as an aggregated measurement of $\lambda(t, \beta)$, where $Y(t, \beta) = \beta\lambda(t, \beta)g$.
  \begin{figure}[htbp]
    \centering
    \includegraphics[width=0.6\linewidth]{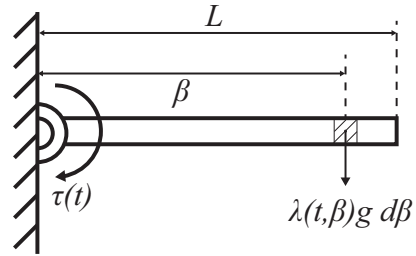}
    \caption{A demonstration of aggregated measurements by considering the torque applied by a hinged rod.}
    \label{fig: aggregated measurement demo}
  \end{figure}
\end{example}

\begin{example}[Statistical moments]
    In many applications involving ensemble systems, although the state variable $X(t, \beta)$ can be observed, one cannot distinguish each system in the ensemble from another so that the index $\beta$ is no longer available. In these cases, the measurements avaialable to users are the snapshot-type measurements shown in bottom of Figure \ref{fig: aggregated measurement demo statistical moments}. In order to use these snapshots to characterize the dynamics of ensemble systems, one may treat $X(t, \beta)$ as a random variable over $\beta$ and compute the statistical moments of $X(t, \beta)$, i.e.,
    \[
      Y(t) = \int_K X^{\alpha} (t, \beta) \,\mathrm{d}\mathbb{F}\beta,
    \]
    where $\mathbb{F}$ is a prior distribution defined on $K$. If the ensemble is finite, i.e., $K$ consists of finite number of points denoted as $\beta_i, i =1, \ldots, N$, then the prior distribution $\mathbb{F}$ is concentrated at $\beta_i, i = 1, \ldots, N$. If additionally we assume no prior information about how the index parameter $\beta$ distributes is available, then all systems in the ensemble are equally important so that $\mathbb{F}$ should be chosen as a uniform distribution on $K$. Hence the aggregated measurement boils down to
  \begin{equation}
    \label{equ: statistical moments discrete}
    Y(t) = \frac{1}{N} \sum_{i=1}^N X^{\alpha}(t, \beta_i).
  \end{equation}
  \vspace{-3mm}
  \begin{figure}[htbp]
    \centering
    \includegraphics[width=0.75\linewidth]{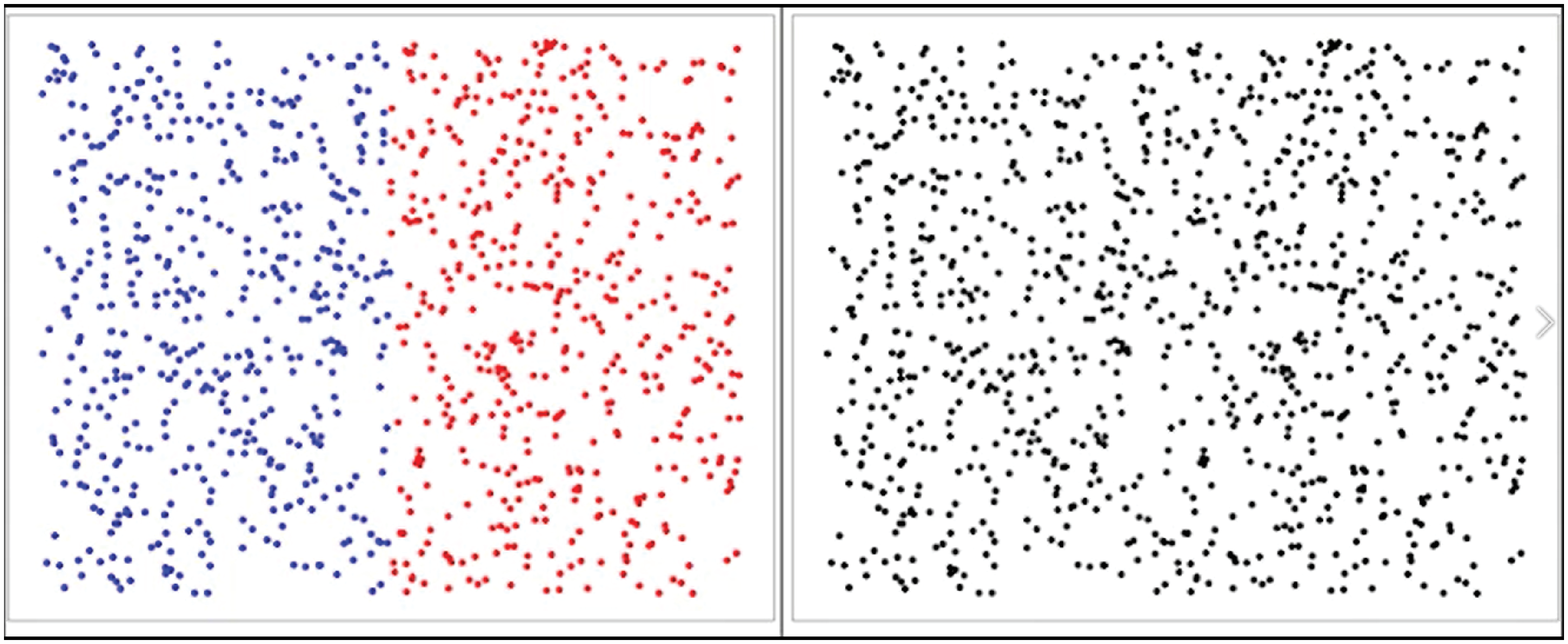}
    \caption{A demonstration of aggregated measurements by computing the statistical moments of snapshots.}
    \label{fig: aggregated measurement demo statistical moments}
  \end{figure}
\end{example}

With the notion of aggregated measurements, typical problems involving ensemble systems include recognizing the dynamics of an ensemble, i.e., $F$ in \eqref{equ: ensemble definition}, and identifying similarity among ensemble systems without knowledge of their dynamics through aggregated measurements.

\subsection{RKHS and Kernel Two-Sample Test} \label{subsec: RKHS MMD}

In this section, we briefly review related concepts of RKHS and kernel two-sample test \cite{gretton12a}, which will be used for recognizing and clustering ensemble systems with unknown dynamics. 

Let us consider an arbitrary non-empty set $\mathcal{X}$ and let $\mathcal{H}$ be a Hilbert space of real-valued functions on $X$. We define the Reproducing Kernel Hilbert Space and its reproducing kernel as follows:

\begin{definition}
  \label{def: RKHS}
  A Hilbert space $\mathcal{H}$ of functions $f: \mathcal{X}\to \mathbb{R}$ is called a \emph{Reproducing Kernel Hilbert Space} (RKHS) if for all $x\in X$, the evaluation functional $L_{x}: \mathcal{H}\to \mathbb{R}$, $f\mapsto f(x)$ is a bounded operator on $\mathcal{H}$.
\end{definition}


\begin{definition}
  \label{def: kernel}
	Let $\mathcal{X}$ be non-empty. A function $\rho:\mathcal{X}\times \mathcal{X}\to \mathbb{R}$ is called a kernel on $\mathcal{X}$ if it is continuous and symmetric, i.e., $\rho(x_1,x_2) = \rho(x_2, x_1)$ for all $s_1, s_2\in \mathcal{X}$. Furthermore, a kernel is said to be positive definite if $\sum_{i,j=1}^n c_{i}c_{j}\rho(x_{i},x_{j})\geqslant 0$ for any $n\in \mathbb{Z}_{+}$, $c_1, \ldots, c_n\in \mathbb{R}$, and $x_1, \ldots, x_n\in \mathcal{X}$.
\end{definition}

\begin{definition}[reproducing kernel]
  \label{def: reproducing kernel}
  Let $\mathcal{H}$ be an RKHS. A kernel $\rho: \mathcal{X}\times \mathcal{X}\to \mathbb{R}$ on $\mathcal{X}$ is called a \emph{reproducing kernel} of $\mathcal{H}$ if
  \begin{enumerate}[label=(\roman*)]
    \item for any $x\in \mathcal{X}$, $\rho( \cdot , x) \in \mathcal{H}$;
    \item for any $x\in \mathcal{X}$ and $f\in \mathcal{H}$, $f(x) = \langle f( \cdot ), \rho(\cdot , x) \rangle_{\mathcal{H}}$.
  \end{enumerate}
\end{definition}

It is easy to show that, by the Riesz representation theorem, every RKHS possesses a reproducing kernel. Conversely, we can also define an RKHS in terms of its reproducing kernel. The following theorem, known as the Moore-Aronszajn theorem, reveals the intimate connections among notions of positive definite kernel, reproducing kernel, and RKHS.

\begin{theorem}[Moore-Aronszajn]
	Let $\mathcal{X}$ be non-empty and $\rho:\mathcal{X}\times \mathcal{X}\to \mathbb{R}$ be a positive definite kernel. There exists a unique RKHS with $\rho$ being its reproducing kernel.
\end{theorem}

One of the most significant results stemming from the concept of RKHS is the so-called ``representer theorem'' \cite{kimeldorf1971}, which states that the minimizer of the infinite-dimensional least-squared problem in an RKHS can be characterized by evaluating its reproducing kernel on a finite set of points. However, given an RKHS $\mathcal{H}$, it is in general difficult to identify its associated reproducing kernel. Nevertheless, owing to the Moore-Aronszajn theorem, most applications of RKHS start from a positive definite kernel $\rho$, and then solves a minimization problem over the unique RKHS induced by $\rho$, circumventing the challenge of computing the reproducing kernel of a given RKHS.

Besides the famous representer theorem, the past decade has witnessed the success of ``kernel two-sample test'' \cite{gretton12a} in an RKHS to identify the similarity between two probability distributions $\mathbb{P}$ and $\mathbb{Q}$ defined on $\mathcal{X}$ using i.i.d. samples from them. Specifically, given a reproducing kernel $\rho( \cdot , \cdot ) : \mathcal{X}\times \mathcal{X}\to \mathbb{R}$ and its induced RKHS $\mathcal{H}$, we define the \emph{kernel mean embedding} of a probability distribution $\mathbb{P}$, denoted by $\mu_\mathbb{P} \in \mathcal{H}$, s.t.
\[
  \mu_\mathbb{P}=\int\rho(\cdot, x)\,\mathrm{d}\mathbb{P}(x).
\]
  Through the kernel mean embedding, the similarity between $\mathbb{P}$ and $\mathbb{Q}$ can be measured by the \emph{maximum mean discrepancy} (MMD) defined by
\[
  \mathrm{MMD}^2(\mathbb{P}, \mathbb{Q}, \rho) = \|\mu_{\mathbb{P}} - \mu_{\mathbb{Q}}\|_{\mathcal{H}},
\]
which has an unbiased empirical estimate using i.i.d.\ samples $\{x_k\}_{k=1}^{n_1}$ and $\{y_k\}_{k=1}^{n_2}$ from $\mathbb{P}$ and $\mathbb{Q}$, respectively, given by
\begingroup
\allowdisplaybreaks
\begin{multline*}
  \widehat{\mathrm{MMD}}^2 (\mathbb{P}, \mathbb{Q}, \rho) = \frac{1}{n_1(n_1-1)} \sum_{i=1}^{n_1}\sum_{\substack{j=1\\ i\neq j}}^{n_1} \rho(x_i, x_j) + \\
  \frac{1}{n_2(n_2-1)}\sum_{i=1}^{n_2} \sum_{\substack{j=1\\ i\neq j}}^{n_2} \rho(y_i, y_j) - \frac{2}{n_1n_2} \sum_{i=1}^{n_1}\sum_{j=1}^{n_2} \rho(x_i , y_j).
\end{multline*}
\endgroup

It is proved in \cite{Cortes08} that when $\rho$ is a universal kernel, e.g., Gaussian-RBF-type kernel (see Definition~\ref{def: universal kernel} and Theorem~\ref{thm: universality.GRBF} in Appendix), then $\mathrm{MMD}^2(\mathbb{P}, \mathbb{Q}, \rho) = 0$ if and only if $\mathbb{P} = \mathbb{Q}$. This result lays the foundation for the framework we introduce in the next section, which allows us to distinguish or identify distributions by computing the $\widehat{\mathrm{MMD}}$ from their sample datasets using a universal kernel.

\section{Methodology}
\label{sec: method}

In this section, we establish the framework of recognizing an ensemble with unknown dynamics through aggregated measurements. The key point in this framework is to compare the distribution of aggregated measurements of the unknown ensemble with a baseline model in an RKHS by perturbing the ensemble system using random control signals. Leveraging the newly introduced concept of ``aggregated Markov parameter'', we also provide a systematic approach to compute the best linear baseline model for ensemble recognition. Furthermore, we demonstrate that the proposed framework can be generalized to cluster multiple ensembles with unknown dynamics.

\subsection{Recognition of Ensembles Using a Baseline Model}

To fix ideas, let us consider an ensemble index by $\beta\in K\subset \mathbb{R}$, where $K$ is compact, given by
\begin{equation}
  \label{equ: ensemble aggregated}
  \begin{aligned}
    \Sigma:\ \frac{\mathrm{d}}{\mathrm{d}t} &X(t, \beta)=F(t, \beta, X(t, \beta), u(t)), \\
    &X(0, \beta)=X_0(\beta), \\
    &Y(t)=\int_{K}G(t, \beta, X(t, \beta), u(t)) \mathrm{d}\beta,
  \end{aligned}
\end{equation}
where $X(t, \cdot)$ is the state; $u(t) \in \mathcal{U}\subset C([0, T], \mathbb{R}^m)$ with $\mathcal{U}$ being closed; $F$ is an unknown $\mathbb{R}^n$-valued smooth function; $G$ is a $\mathbb{R}^q$-valued smooth function, called the ``aggregated observation function'', yielding the aggregated measurement $Y(t)$.

In this work, we assume that $Y$ is in $\in C_1([0, T], \mathbb{R}^q)$, the space of continuously differentiable $\mathbb{R}^q$-valued functions over $[0, T]$, such that there exists $C>0$ satisfying $\|Y(t)\|_{L^{\infty}} < C$ and $\|\dot{Y}(t)\|_{L^{\infty }} < C$ for all $Y$. Hence, by the Arzel\`{a}-Ascoli theorem, the set of all possible aggregated measurements, denoted as
\begin{equation}
  \label{equ: possible.aggregated.measurements}
	\mathcal{F} := \{Y(t)\in C_1([0, T], \mathbb{R}^q): u(t) \in \mathcal{U}\}, 
\end{equation}
is a compact subset of $C_1([0, T], \mathbb{R}^q)$ under the uniform norm. Therefore, $\mathcal{F}$ is also compact as a subset of $L_2([0, T], \mathbb{R}^q)$.

Next, we discuss the recognition of an unknown ensemble using a baseline model. Assume the control signal $u(t)$ of the ensemble in \eqref{equ: ensemble aggregated} is subject to a distribution $\mathbb{U}$ defined on $\mathcal{U}$. Then, the aggregated measurement $Y(t)$ of $\Sigma$ is subject to a distribution defined on $\mathcal{F}$, denoted as $\mathbb{P}$, which is determined by the system dynamics of $\Sigma$, the aggregated observation function $G$, and the distribution of control inputs $\mathbb{U}$. Given a baseline ensemble $\Gamma$, whose distribution of aggregated measurements is denoted by $\mathbb{Q}$, we can recognize the unknown ensemble $\Sigma$ through detecting the similarity between $\mathbb{P}$ and $\mathbb{Q}$.

In particular, let us consider a reproducing kernel function $\rho(\cdot, \cdot): \mathcal{M} \times \mathcal{M} \to \mathbb{R}$ and its induced RKHS $\mathcal{H}$. Without loss of generality, we assume the two distributions $\mathbb{P}$ and $\mathbb{Q}$ over $\mathcal{F}$ satisfy
\[
  \int_{\mathcal{F}}\rho^{\frac{1}{2}}(s, s)\, \mathrm{d} \mathbb{P}(s) <\infty \ \text{and}\ \int_{\mathcal{F}}\rho^{\frac{1}{2}}(s, s)\, \mathrm{d} \mathbb{Q}(s) <\infty,
\]
so that the kernel mean embeddings of $\mathbb{P}$ and $\mathbb{Q}$ are well-defined (see Lemma~\ref{lem: existence of kernel mean embedding} in Appendix). Hence, the MMD in $\mathcal{H}$ between $\mathbb{P}$ and $\mathbb{Q}$ can be evaluated numerically using i.i.d.\ samples from these two distributions. Specifically, we first draw independent random control signals $\{u^{(k)}\}_{k=1}^{n_1}$ that are subject to a fixed control input distribution $\mathbb{U}$, and then apply these control signals to the unknown ensemble $\Sigma$. The perturbations $u^{(k)}$ result in a collection of aggregated measurements $S_1 := \{Y_1^{(k)}(t)\}_{k=1}^{n_1}$, which are i.i.d.\ samples of the distribution $\mathbb{P}$. Similarly, we perform the same procedure to the baseline ensemble $\Gamma$: we first draw independent control signals $\{u^{(k)}\}_{k=1}^{n_2}$ that are subject to the same distribution $\mathbb{U}$, and then apply them to $\Gamma$ to obtain another collection of aggregated measurements $S_2:=\{Y_2^{(k)}(t)\}_{k=1}^{n_2}$, which are i.i.d.\ samples of the distribution $\mathbb{Q}$. Then, the empirical MMD between $\mathbb{P}$ and $\mathbb{Q}$ computed using $S_1$ and $S_2$, denoted as $h(S_1, S_2)$, is given by
\begin{equation}
  \label{equ: empirical MMD}
  \begin{aligned}
    h(S_1, &S_2):=\frac{1}{n_1(n_1-1)}  \sum_{i=1}^{n_1} \sum_{\substack{j = 1\\j\neq i}}^{n_1}\rho(Y_1^{(i)}(t), Y_1^{(j)}(t)) \\
      &+ \frac{1}{n_2(n_2-1)}\sum_{i=1}^{n_2}\sum_{\substack{j = 1\\j\neq i}}^{n_2} \rho(Y_2^{(i)}(t), Y_2^{(j)}(t)) \\
      & - \frac{2}{n_1n_2} \sum_{i=1}^{n_1} \sum_{j=1}^{n_2} \rho(Y_1^{(i)}(t), Y_2^{(j)}(t)),     
  \end{aligned}  
\end{equation}
which is an unbiased estimate of $\mathrm{MMD}^2(\mathbb{P}, \mathbb{Q}, \rho)$. Since $\mathrm{MMD}^2(\mathbb{P}, \mathbb{Q}, \rho) = 0$ implies $\mathbb{P} = \mathbb{Q}$, given that $\rho$ is a universal kernel, we call the unknown ensemble $\Sigma$ \emph{dynamically equivalent} to the baseline ensemble $\Gamma$ if $h(S_1, S_2) = 0$.

In practice, we may run a test on the null-hypothesis of $\mathbb{P} = \mathbb{Q}$. Without loss of generality, we assume $n_1 = n_2 = I$. Then, we have the following corollary on the number of sampled trajectories to falsely reject the null-hypothesis with low probability.

\begin{corollary}
	Given a non-empty set $\mathcal{F}$ of aggregated measurements, a reproducing kernel $\rho:\mathcal{F}\times \mathcal{F} \to \mathbb{R}$ satisfying $0 < \rho(x, y) < C$ for some constant $C>0$ and $x, y\in \mathcal{F}$, and two collections of i.i.d.\ samples $S_1 := \{x_k\}_{k=1}^I$ and $S_2 := \{y_k\}_{k=1}^I$ from two independent distributions $\mathbb{P}$ and $\mathbb{Q}$ defined on $\mathcal{F}$, respectively, the number of samples required, i.e., $I$, to falsely reject the null-hypothesis $\mathbb{P} = \mathbb{Q}$ with probability lower than $\alpha $ up to $\epsilon$-error, is given by
  \[
	 	I > -\frac{16C^2}{\epsilon^2} \ln \alpha.
	\]
	Furthermore, given a fixed $I > 0$, the acceptance region of $\mathbb{P} = \mathbb{Q}$ at significance level of $\alpha$, is given by
	\begin{equation}
    \label{equ: MMD hypothesis test criterion}
	  h(S_1, S_2) \leq 4C\sqrt{\frac{-\ln \alpha}{I}}.
	\end{equation}
\end{corollary}

\begin{proof}
	As proved in \cite{hoeffding1994probability}, it holds that
	\[
		\mathrm{Pr} (h(S_1, S_2) > \epsilon) \leq \exp\Bigl(-\frac{\epsilon^2 I}{16C^2}\Bigr),
	\]
	given the null-hypothesis $\mathbb{P} = \mathbb{Q}$. Let
  \[
    \mathrm{Pr}(h(S_1, S_2) > \epsilon) \leq \exp\Bigl(-\frac{\epsilon^2 I}{16C^2}\Bigr) < \alpha,
  \]
  then we have $I > -\tfrac{16C^2}{\epsilon^2} \ln \alpha$ when $\epsilon$ and $\alpha$ are fixed; and $\epsilon > 4C \sqrt{(-\ln \alpha)/I}$ when $I$ and $\alpha $ are fixed. Therefore, we reject the null-hypothesis of $\mathbb{P} = \mathbb{Q}$ at a significance level of $\alpha $ when $h(S_1, S_2) > 4C \sqrt{(-\ln \alpha)/I}$, which yields the acceptance region as in \eqref{equ: MMD hypothesis test criterion}.
\end{proof}

\subsection{Aggregated Markov Parameters and Linear Ensemble Approximation}
\label{subsec: Aggregated Markov Parameters and Linear Ensemble Approximation}

In the last subsection, we have introduced a highly\hyp{}conceptualized framework to recognize an unknown ensemble system $\Sigma$ with aggregated measurements by comparing it with a baseline ensemble $\Gamma$ through computing MMD in an RKHS. However, this framework requires us to make an educated guess on the baseline model in order to accurately recognize the unknown ensemble. Therefore, it is critical to devise a systematic approach to adaptively update the baseline model, and thus improves the recognition results. In this subsection, we propose a novel concept of \emph{aggregated Markov parameters}, and provide a tractable procedure to design linear baseline models based on the flow of aggregated Markov parameters.

To be specific, given an unknown ensemble $\Sigma$ in \eqref{equ: ensemble aggregated} and a collection of random control signals $\{u^{(k)}(t)\}_{k=1}^N \subset \mathcal{U}$, we aim to find $A(\cdot )\in C(K, \mathbb{R}^{n\times n})$, $B(\cdot )\in C(K, \mathbb{R}^{n\times m})$, and $C(\cdot )\in C(K, \mathbb{R}^{q \times n})$ to construct a baseline model in the following form,
\begin{equation}
  \label{equ: linear baseline model}
  \begin{aligned}
    \Gamma:\ \frac{\mathrm{d}}{\mathrm{d}t} &Z(t, \beta) = A(\beta)Z(\beta) + B(\beta)u(t), \\
      &Y(t) = \int_{K} C(\beta) Z(t, \beta) \mathrm{d}\beta,
  \end{aligned}
\end{equation}
which \emph{minimizes} $h(S_1, S_2)$ in \eqref{equ: empirical MMD}. Without loss of generality, we assume $Z(0, \beta) = 0$ for all $\beta\in K$. Hence, from the theory of linear system, the state trajectory of the linear ensemble in \eqref{equ: linear baseline model} under the control input $u^{(i)}(t)$ is given by
\begin{align*}
  Z^{(i)}(t, \beta) &= \int_0^t e^{A(\beta)(t-s)} B(\beta)u(s) \,\mathrm{d}s \\
    &= \int_0^t \sum_{j=0}^{\infty} \frac{A^j(\beta)(t-s)^j}{j!} B(\beta) u^{(i)}(s) \,\mathrm{d}s\\
    &= \sum_{j=0}^{\infty} \biggl\{\frac{1}{j!}A^j(\beta) B(\beta) \int_0^t (t-s)^j u^{(i)}(s) \,ds\biggr\}.
\end{align*}
So the aggregated measurement $Y^{(i)}(t)$ of $\Gamma$, under the control input $u^{(i)}(t)$, is given by
\begin{align*}
    Y_2^{(i)}&(t) = \int_K C(\beta)Z^{(i)}(t, \beta)\mathrm{d} \beta\\
    = &\sum_{j=0}^{\infty } \biggl\{\frac{1}{j!} \int_K C(\beta)A^j(\beta) B(\beta) \mathrm{d} \beta\int_0^t (t-s)^j u^{(i)} \,\mathrm{d}s \biggr\}.
\end{align*}

From the aggregated measurement $Y_2^{(i)}(t)$ above, we define the aggregated Markov parameters as follows.
\begin{definition}
  \label{def:aggregated.Markov}
  The \emph{aggregated Markov parameters} $\eta_j$ ($j\in \mathbb{N}$), for the baseline model $\eta$ in \eqref{equ: linear baseline model} is defined as
  \[
    \eta_j := \frac{1}{j!} \int_K C(\beta)A^j(\beta) B(\beta)\,\mathrm{d}\beta \in \mathbb{R}^{q\times m}
  \]
\end{definition}

If we denote $\Lambda_{ij}(t) := \int_0^t (t-s)^j u^{(i)}\,\mathrm{d}s$, then the aggregated measurement is written as $Y_2^{(i)}(t) = \sum_{j=0}^{\infty }\eta_j \Lambda_{ij}(t)$. We observe that $\Lambda_{ij}(t)$ ($i = 1, \ldots , N$ and $j\in \mathbb{N}$) is uniquely determined by $u^{(i)}(t)$ and independent from the dynamics of $\Gamma$. Therefore, the design of $A(\cdot ), B(\cdot )$, and $C(\cdot )$ boils down to the design of aggregated Markov parameters, i.e., $\eta_j$'s.

For ease of exposition, we assume that $m = q = 1$, since the multidimensional case where $m > 1$, $q>1$ can be addressed by designing each component of $\eta_j$. If we denote the sequence $\{\eta_j\}_{j=1}^{\infty}$ by $\eta$, and use $\Gamma(\eta)$ to denote the ensemble in \eqref{equ: linear baseline model} satisfying $\eta_j := \tfrac{1}{j!} \int_K C(\beta)A^j(\beta) B(\beta) \,\mathrm{d}\beta$, then by Lemma~\ref{lem: AMP.in.l2} in Appendix, it holds that $\eta = \{\eta_j\}_{j=0}^{\infty } \in \ell_2$ is square\hyp{}summable.

To design a collection of aggregated Markov parameters $\eta$ minimizing the MMD between $\Sigma$ and $\Gamma(\eta)$, we consider a gradient flow on $\eta$. Specifically, let $\tilde{h}(\eta) = h(S_1, S_2(\eta))$, where $S_2(\eta)$ denotes the collection of aggregated measurements of $\Gamma(\eta)$ and $h(S_1, S_2)$ is defined as in \eqref{equ: empirical MMD}. We compute a flow of aggregated Markov parameters $\eta(s)$ by numerically simulating the differential equation given by 
\begin{equation}
  \label{equ: gradient flow}
  \frac{\mathrm{d}}{\mathrm{d}s}\eta(s) = - \nabla \tilde{h}(\eta(s)), \quad \eta(0) = \theta\in \ell_2,
\end{equation}
where $\nabla \tilde{h}(\eta(s))$ denotes the gradient of $h$ with respect to $\eta$. Then, the MMD between $\Sigma$ and $\Gamma(\eta(s))$ is always non-increasing along the flow $\eta(s)$ since
\[
  \frac{\mathrm{d}}{\mathrm{d}s} \tilde{h}(\eta(s)) (\nabla \tilde{h})^{\intercal} \frac{\mathrm{d} \eta}{\mathrm{d}s}= - \|\nabla \tilde{h}\|_{\ell^2}^2 \leq 0.
\]
Therefore, $\eta^{\ast} = \lim_{s\to \infty}\eta(s)$ minimizes $\tilde{h}(\eta)$ to a local minimum. Although there is no guarantee that $\eta^{\ast}$ is a global minimizer of $\tilde{h}(\eta)$ since $\tilde{h}$ is in general not convex w.r.t.\ $\eta$, one can choose different initial conditions of \eqref{equ: gradient flow}, i.e., $\theta$, and simulate the aggregated flow for multiple times to obtain a better local minimum.

Now suppose that we have obtained a collection of aggregated Markov parameters $\eta^*_j$, $j\in \mathbb{N}$ such that $\tilde{h}(\eta^*)$ is sufficiently small, we can construct $A(\cdot) \in \mathbb{R}^{n\times n}$, $B(\cdot)\in \mathbb{R}^{n\times m}$, and $C(\cdot)\in \mathbb{R}^{q\times n}$ satisfying
\[
  \eta_j^* = \tfrac{1}{j!} \int_K C(\beta)A^j(\beta)B(\beta)\,\mathrm{d}\beta
\]
to realize a linear ensemble system as in \eqref{equ: linear baseline model} that approximates the unknown ensemble $\Sigma$. Indeed, there exists infinitely many linear ensemble realizations for a given set of aggregated Markov parameters, so it depends on one's prior knowledge to determine the dimension of the baseline linear ensemble, i.e., $n$. It is worthwhile to mention that when $\eta_j$'s are scalar, the minimal linear ensemble realization is always $1$-dimensional, i.e., $n = 1$. Specifically, we can set $B(\cdot) = C(\cdot) \equiv  1$, and then have $\eta_j^{\ast} = \int_{K} \tfrac{1}{j!} A^j(\beta) \,\mathrm{d}\beta$. Therefore, we can recover $A(\cdot)$ using $\eta_j^{\ast}$'s by constructing the moment generating function (MGF) of $A(\cdot)$ since
\begin{align*} 
  \sum_{j=0}^{\infty } \eta_j^* s^j &= \int_{K}\biggl(\sum_{j=0}^{\infty } \frac{1}{j!}A^j(\beta)s^j\biggr) \,\mathrm{d}\beta \\
    &= \int_K e^{A(\beta)s} \mathrm{d} \beta = \mathrm{MGF}_A(s),
\end{align*}
where the power series expansion of the moment generating function is valid as a result of Lemmas~\ref{lem: AMP.in.l2} and~\ref{lem: MGF expansion} (see Appendix).

\subsection{Clustering of Ensemble Systems}

In previous sections, we have introduced an RKHS-based ensemble recognition method to recognize whether an unknown ensemble is dynamically equivalent to a baseline ensemble by computing the MMD between the aggregated measurement distributions of two ensembles. Leveraging the idea of computing pairwise MMDs, the proposed RKHS-based approach can be extended for the purpose of clustering multiple unknown ensembles. 

In Figure~\ref{fig: ensemble clustering demo}, we illustrate the idea of clustering 4 ensembles of unknown ensembles, which can be generalized to clustering arbitrary number of ensemble systems. Specifically, we generate random control signals under a distribution $\mathbb{U}$, and apply them to all the $4$ ensembles, yielding $4$ distributions of aggregated measurements, denoted as $\mathbb{P}, \mathbb{Q}, \mathbb{S}, \mathbb{W}$, respectively. As mentioned in Section~\ref{subsec: RKHS MMD}, given a universal reproducing kernel $\rho( \cdot , \cdot ): \mathcal{F} \times \mathcal{F}\to \mathbb{R}$ and its induced RKHS $\mathcal{H}$, there exists a bijiection between the space of distributions over $\mathcal{F}$ and the RKHS $\mathcal{H}$. Hence, through the kernel mean embedding, the distributions $\mathbb{P}$, $\mathbb{Q}$, $\mathbb{S}$, $\mathbb{W}$ associated with the $4$ unknown ensembles are mapped as $4$ elements in $\mathcal{H}$, i.e., $\mu_{\mathbb{P}}$, $\mu_{\mathbb{Q}}$, $\mu_{\mathbb{S}}$, $\mu_{\mathbb{W}}$, respectively. In the RKHS $\mathcal{H}$, the distance between two kernel mean embeddings can be characterized by computing MMD using the corresponding aggregated measurements. Therefore, we can use any metric-based clustering methods, such as the agglomerative hierarchical clustering, to cluster the unknown ensembles by clustering their kernel mean embeddings in RKHS.

\begin{figure}[htbp]
  \centering
  \includegraphics[width=0.85\linewidth]{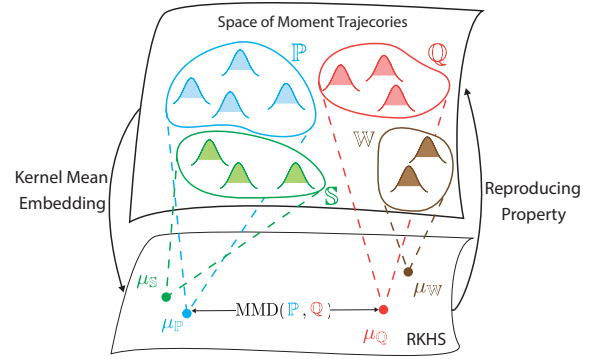}
  \caption{A demonstration of clustering 4 ensembles with unknown dynamics in an RKHS.}
  \label{fig: ensemble clustering demo}
\end{figure}

\section{Numerical Experiments}
\label{sec: experiments}

In this section, we provide various examples to illustrate the efficacy and robustness of the proposed ensemble recognition and clustering approaches in RKHS. If not specified, we select the aggregated measurements to be the statistical moments of the state trajectory $X(t, \cdot )$. Specifically, the aggregated measurements of an experiment take the form of $Y(t) = (Y_1(t), \ldots, Y_q(t))^{\intercal} \subset  \mathcal{F}$, where $\mathcal{F}$ is the set of all possible aggregated measurements as defined in \eqref{equ: possible.aggregated.measurements}, and $Y_{\alpha}(t)$, $\alpha = 1, \ldots, q$, is the $\alpha$\textsuperscript{th} moment of $X(t, \cdot)$, given by
\[
  Y_{\alpha}(t) = \frac{1}{|K|} \int_K X^\alpha (t, \beta) \,\mathrm{d}\beta.
\]
Here $X^{\alpha}(t, \beta) = (x_1^\alpha(t, \beta), \ldots , x_n^{\alpha}(t, \beta))^{\intercal}$, and $x_j(t, \beta)$, $j =1, \ldots , n$ are the $j$\textsuperscript{th} component of $X(t, \beta)$.

When computing the MMDs, we select the reproducing kernel $\rho: \mathcal{F}\times \mathcal{F} \to \mathbb{R}$ as the Gaussian-RBF kernel, that is, a kernel which satisfies that for any $x, y$ in the RKHS $\mathcal{H}$ induced by $\rho$, we have
\begin{equation}
  \label{eq:universal.kernel.experiments}
  \rho(x, y) = \mathrm{exp} \biggl( -\sigma \sum_{\alpha=1}^q \int_{0}^{T} \|x_{\alpha}(s) - y_{\alpha }(s)\|_2^2 \,\mathrm{d}s\biggr),
\end{equation}
where $\sigma>0$ and $\| \cdot \|_2$ denotes the $2$-norm in $\mathbb{R}^n$. As a result of Theorem \ref{thm: universality.GRBF} in Appendix, we have $\rho$ to be a universal reproducing kernel. Furthermore, since $0<\rho(x, y)<1$ for all $x, y\in \mathcal{H}$, the constant $C$ in \eqref{equ: MMD hypothesis test criterion} can be selected as $C = 1$. 


\begin{example}[Recognizing an ensemble using different numbers of moment trajectories]
  \label{ex:recognizing.an.ensemble.different.trajectories}
  In this example, we consider the ensemble of $2$-dimensional harmonic oscillators, given by
  \begin{equation}
    \label{equ: ex1 harmo ensemble}
    \frac{\mathrm{d}}{\mathrm{d}t}X(t, \beta)=
    \begin{bmatrix} 0 & -\beta \\ \beta & 0 \end{bmatrix}
    X(t, \beta)+
    \begin{bmatrix} u \\ v \end{bmatrix},\ \beta\in K,
  \end{equation}
  where $X(t, \cdot ) \in L^2(K, \mathbb{R}^2)$ is the state, and $u(t),\ v(t) :[0, T]\to \mathbb{R}$ are piecewise constant control signals.

  We aim to characterize the numerical performance of the proposed RKHS-based ensemble recognition approach with respect to different numbers of sampled moment trajectories. Specifically, we choose $K = [-1, 1]$, $T = 1$, and fix the initial condition of the ensemble of harmonic oscillators in \eqref{equ: ex1 harmo ensemble} as $X(0, \beta) = (1, 0)^{\intercal}$ for all $\beta \in K$. We then generate random control signals $\{u^{(k)}(t)\}$ and $\{v^{(k)}(t)\}$, $k=1, \ldots, 1000$, where each signal is discretized by a time-step of $0.02$ with the value at each time-step randomly drawn from $[-5, 5]$ under a uniform distribution. Then, we run $1000$ independent experiments by applying the control signals $u^{(k)}$, $v^{(k)}$, $k = 1, \ldots , 1000$ on the ensemble in \eqref{equ: ex1 harmo ensemble}. To collect the aggregated measurements, we randomly drew $300$ points from $[-1, 1]$ under a uniform distribution, denoted as $\beta_i$, $i = 1, \ldots, N$, then computed the first $3$ of statistical moments in each experiment using $Y_{\alpha }(t) = \tfrac{1}{300}\sum_{i=1}^{300} X^{\alpha}(t, \beta_i)$, $\alpha = 1, 2, 3$.

  We denote $S_0:=\{Y^{(k)}(t)\}_{k=1}^{1000}$ as the collection of aggregated measurements in all $1000$ experiments, and $S_0|_I$ as a subset of $S_0$ consists of $I$ random samples drawn from $S_0$ under a uniform distribution. We compute the empirical MMD between $S_0|_{I_1}$ and $S_0|_{I_2}$ with $I_1 = I_2 = I$, and varied $I$ from $1$ to $1000$. Figure~\ref{fig: MMD_vary_N} demonstrates the results of the computed empirical MMD, where the acceptance region si obtained from \eqref{equ: MMD hypothesis test criterion}, and the constant parameter in the reproducing kernel is taken as $\sigma = 1\times 10^{-3}$. Each point on the solid blue curve represents the mean plus/minus the standard deviation of $50$ independent experiments. In this experiment, we observe from Figure~\ref{fig: MMD_vary_N} that the null-hypothesis that two subsets $S_0|_{I_1}$ and $S_0|_{I_2}$ were generated from the same ensemble should be accepted at a significance level of $\alpha = 0.05$, when the sample of moment trajectories is decently large, for example $I \geq 10$.
\end{example}

\begin{figure}[htbp]
  \centering
  \includegraphics[width=0.95\linewidth]{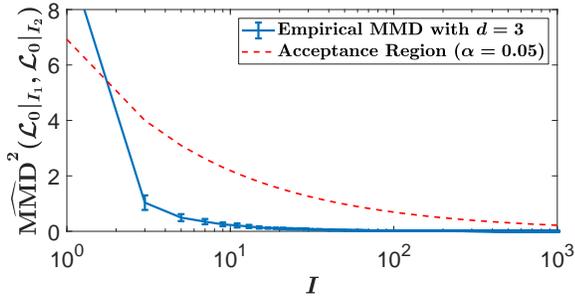} 
  \caption{Results of recognizing the same ensemble using various numbers of moment trajectories. Each point on the solid blue line is the mean plus/minus the standard deviation of $50$ independent experiments.}
  \label{fig: MMD_vary_N}
\end{figure}

\begin{example}[Recognizing an ensemble using different sets of sampled systems]
  \label{ex:recognizing.different.sets}
  
  The purpose of this example is to show that the proposed RKHS-based recognition approach is robust against different sample of $\beta_i$'s. Similar to the previous example, we apply the control signals $u^{(k)}$, $v^{(k)}$, $k =1, \ldots , 1000$ to the ensemble of harmonic oscillators in \eqref{equ: ex1 harmo ensemble}. Then, we construct two other ensembles of harmonic oscillators by changing the index range to be $K = [-1, 1]$ and $K = [-2, 2]$, and computed their statistical moments of order $1$ to $d$ as the aggregated measurements, which we denot as $S_1$ and $S_2$, respectively. In this case, the null-hypothesis of $S_i$ and $S_j$, $i, j = 0,1,2$ generated from the same ensemble should be rejected at a significance level of $\alpha = 0.05$ if $\widehat{\mathrm{MMD}}^2(S_i, S_j) > \tfrac{4\times 1}{\sqrt{1000}}\sqrt{-\ln 0.05} = 0.219$.

  The pairwise MMDs among $S_0$, $S_1$, and $S_2$ are shown in Table~\ref{tab: same ensemble with different sampled systems}, where the constant parameter in the reproducing kernel was selected as $\sigma = 1\times 10^{-3}$. As we shall observe from Table \ref{tab: same ensemble with different sampled systems}, when only considering the $1$\textsuperscript{th}-order moment, i.e., $d = 1$, all three ensembles were recognized as the same ensemble; while when higher order of moments were involved, i.e., $d = 2$ or $d = 3$, the proposed RKHS-based approach successfully distinguished the 3 different ensembles, although the samples of $\beta_i$'s were not identical in these 3 cases.
\end{example}

\begin{table*}[htb]
\centering
\renewcommand{\arraystretch}{2}
\begin{tabular}{|c|C{1cm}|C{1cm}|C{1cm}|C{1cm}|C{1cm}|C{1cm}|C{1cm}|C{1cm}|C{1cm}|}
\hline
\multicolumn{10}{|c|}{$\widehat{\mathrm{MMD}}^2(S_i, S_j)$}                \\ \hline
    & \multicolumn{3}{c|}{$d=1$} & \multicolumn{3}{c|}{$d=2$} & \multicolumn{3}{c|}{$d=3$} \\ \hline
\diagbox[height = 0.7cm]{$i$}{$j$} & $0$      & $1$      & $2$      & $0$      & $1$      & $2$      & $0$      & $1$      & $2$      \\ \hline
$0$   & 0.002  & 0.002  & 0.012  & 0.002  & 0.004  & 0.307  & 0.001  & 0.097  & 1.030  \\ \hline
$1$   & 0.002  & 0.002  & 0.012  & 0.004  & 0.002  & 0.291  & 0.105  & 0.001  & 1.034  \\ \hline
$2$   & 0.012  & 0.012  & 0.012  & 0.306  & 0.291  & 0.002  & 1.030  & 1.034  & 0.001  \\ \hline
\end{tabular}
\caption{Results of recognizing ensembles through different sampled systems, through different orders of statistical moments.}\label{tab: same ensemble with different sampled systems}
\end{table*}

Next, we present an example of inferring the aggregated Markov parameters of an ensemble system through an aggregated flow in Section~\ref{subsec: Aggregated Markov Parameters and Linear Ensemble Approximation}.

\begin{example}[Linear ensemble approximation of an unknown ensemble]
  \label{ex:Linear.ensemble.approximation}
  We consider an ensemble of linear systems indexed by $\beta\in [0.5,1]$, with its aggregated measurements given by
  \begin{equation}
    \label{equ: ex.linear.approximation.sys}
    \begin{aligned}
      \frac{\mathrm{d}}{\mathrm{d}t}
      \begin{bmatrix} x_1(t, \beta) \\ x_2(t, \beta) \end{bmatrix} &=
      \begin{bmatrix} 0 & -\beta \\ \beta & 0 \end{bmatrix}
      \begin{bmatrix} x_1(t, \beta) \\ x_2(t, \beta) \end{bmatrix} +
      \begin{bmatrix} u \\ 0 \end{bmatrix}, \\
      y(t) &= \int_{0.5}^1 x_1(t, \beta) \,\mathrm{d}\beta.
    \end{aligned}
  \end{equation}
  where the initial condition is set to be $x_1(t, \beta)=x_2(t, \beta)=0$ for all $\beta\in [0.5, 1]$.
  
  For the above ensemble system, we construct a collection of $100$ random control signals $\{u^{(k)}\}_{k=1}^{100}$, $t\in [0, 1]$, where each signal is discretized by a time-step of $0.001$, and the value at each time step is randomly drawn from $[-2, 2]$ under a uniform distribution. Then, we run $100$ independent experiments by applying $u^{(k)}$, $k=1, \ldots , 100$, and collected the corresponding aggregated measurements, denoted as $y^{(k)}(t)$. To recognize the first 10 aggregated Markov parameters of the above ensemble, we conduct the method of aggregated flow described in \eqref{equ: gradient flow}, and the results are illustrated in Figures~\ref{fig: linear_approximation_eta} and \ref{fig: linear_approximation_MMD}. When simulating the aggregated flow, we apply the first-order approximation of \ref{fig: linear_approximation_MMD}, given by
  \[
    \eta(s_{k+1}) = \eta(s_{k}) - \epsilon \nabla h(\eta(s_k)),
  \]
  with $\epsilon = 0.001$. The actual and estimated Markov parameters, denoted by $\eta_j$ and $\hat{\eta}_{j}$, respectively, after $500$ iterations, i.e., $k = 500$, are shown in Table \ref{tab:linear approximation}. Figure~\ref{fig: linear_approximation_trajs} provides the aggregated measurements collected from the ensemble in \eqref{equ: ex.linear.approximation.sys} and from the linear ensemble constructed using $\hat{\eta}_j$'s, which are denoted by $y^{(k)}(t)$ and $\hat{y}^{(k)}(t)$, $k = 1, 2,3$. For demonstration purpose, only the aggregated measurements of the first 3 experiments are plotted.
\end{example}

\begin{figure}[htbp]
  \centering
  \adjustbox{minipage=1em,valign=t}{\subcaption{}\label{fig: linear_approximation_eta}}%
  \begin{subfigure}[b]{0.45\linewidth}
    \centering\includegraphics[width=1\linewidth]{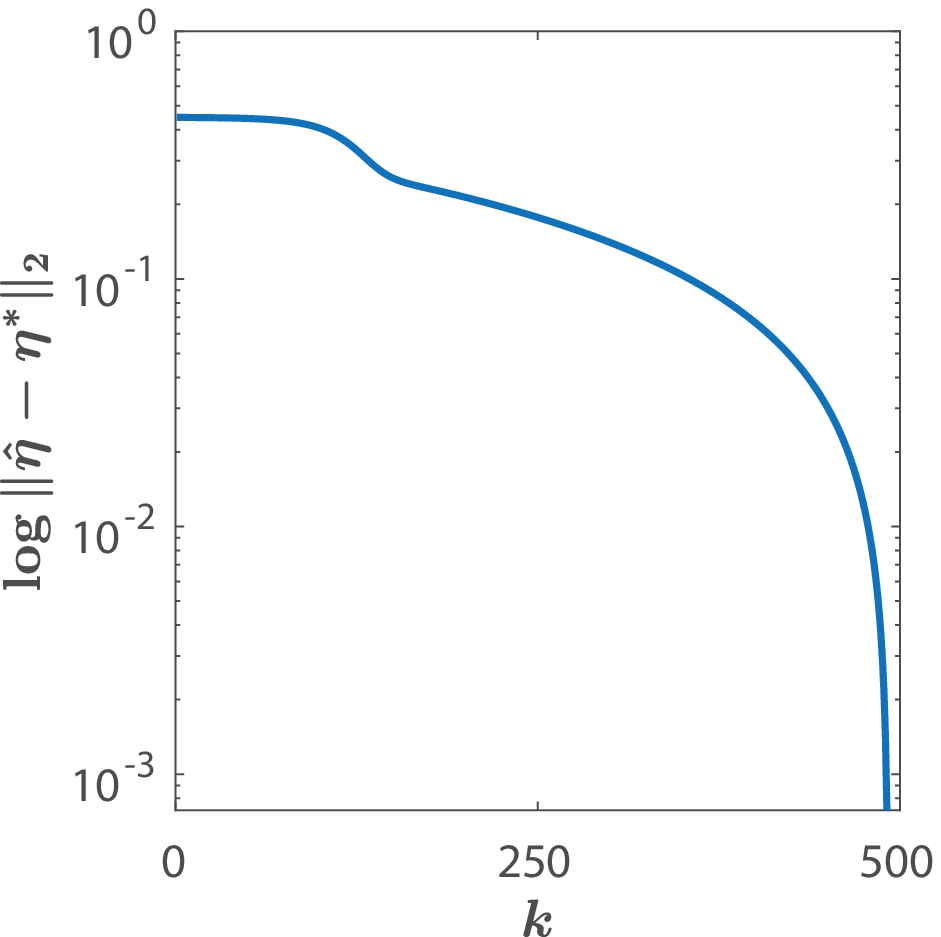}
  \end{subfigure}
  \adjustbox{minipage=1em,valign=t}{\subcaption{}\label{fig: linear_approximation_MMD}}%
  \begin{subfigure}[b]{0.45\linewidth}
    \centering\includegraphics[width=1\linewidth]{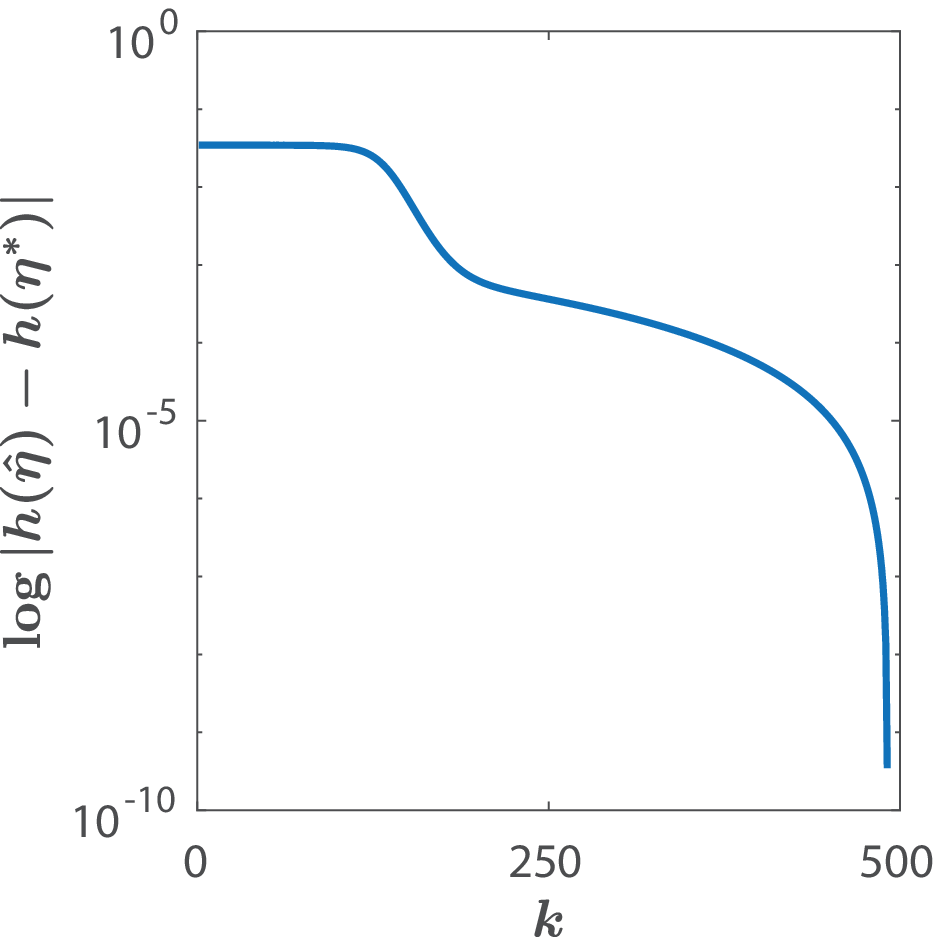} 
  \end{subfigure}

  \adjustbox{minipage=1em,valign=t}{\subcaption{}\label{fig: linear_approximation_trajs}}%
  \begin{subfigure}[b]{0.9\linewidth}
    \centering\includegraphics[width=1\linewidth]{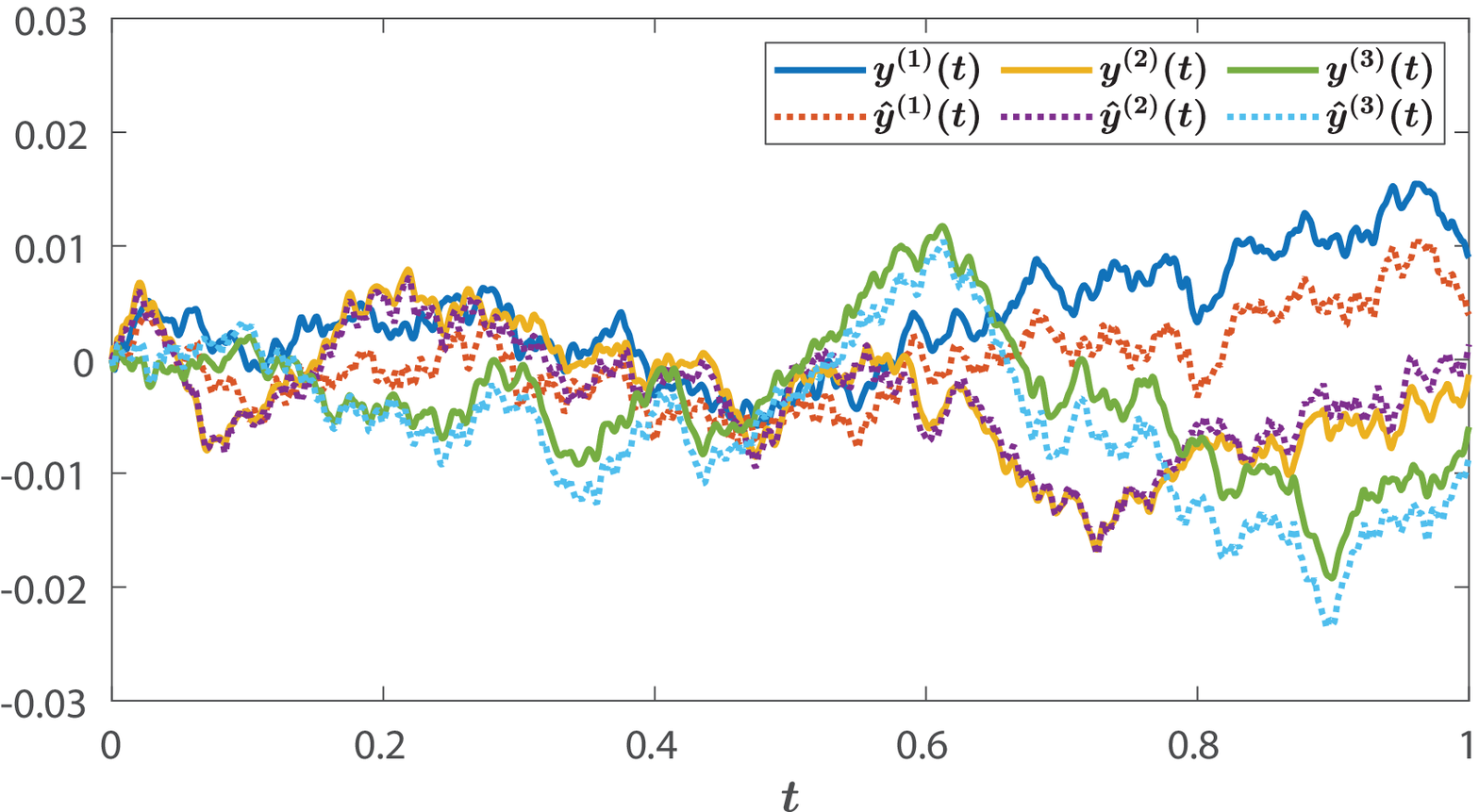} 
  \end{subfigure}
  \caption{Results of linear ensemble approximation.}
\end{figure}

\begin{table}[h]
  \renewcommand{\arraystretch}{1.5}
  \centering
  \begin{tabular}{|c|c|c|c|c|c|c|c|c|c|c|}
  \hline
  $j$ & $0$ & $1$ & $2$ & $3$ & $4$ \\
  \hline
  $\eta_j$ & $0.50$ & $0$ & $-0.15$ & $0$ & $0.01$ \\
  \hline
  $\hat{\eta}_j$ & $0.44$ & $-0.01$ & $-0.05$ & $-0.05$ & $-0.04$ \\
  \hline
  \hline
  $j$ & $5$ & $6$ & $7$ & $8$ & $9$\\
  \hline
  $\eta_j$ & $0$ & $0.00$ & $0$ & $0.00$ & $0$\\
  \hline
  $\hat{\eta}_j$ & $-0.03$ & $-0.03$ & $-0.03$ & $-0.02$ & $-0.02$\\
  \hline
  \end{tabular}
  \caption{The actual and estimated aggregated Markov parameters of the linear ensemble in \eqref{equ: ex.linear.approximation.sys}.}
      \label{tab:linear approximation}
\end{table}

The examples above in this section exhibit the efficacy and robustness of our framework in recognizing an unknown ensemble. In what follows, we will demonstrate the robustness of the proposed approach to recognize and cluster multiple different ensembles. The system dynamics of all the ensembles under consideration are displayed in Table~\ref{tab:Ensemble clustering system dynamics}.

\begin{table*}[h]
  \renewcommand{\arraystretch}{2}
  \centering
  \begin{tabular}{|C{0.3cm}|C{6.5cm}|C{6.5cm}|}
  \hline
  $j$ & System dynamics with $\beta\in K$  & Initial condition \\
  \hline
  1& \multirow{3}{*}{$\renewcommand{\arraystretch}{1.2}\frac{\mathrm{d}}{\mathrm{d}t}X(t, \beta)  = \begin{bmatrix}
    0 & -\beta \\ \beta & 0
  \end{bmatrix} X(t, \beta) + \begin{bmatrix}
    u\\ v
  \end{bmatrix}$} & $X_1(\beta) \equiv (1, 0)^{\intercal}$.  \\ \cline{1-1} \cline{3-3}
  2 & & $X_2(\beta) \equiv (0, 1)^{\intercal}$.  \\ \cline{1-1} \cline{3-3}
  3 & &  Random in $[0, 1]\times [0, 1]$ for each $\beta \in K$. \\
  \cline{1-3}
  4& \multirow{3}{*}{$\renewcommand{\arraystretch}{1.2}\frac{\mathrm{d}}{\mathrm{d}t}X(t, \beta)  = \begin{bmatrix}
    0 & -2\beta \\ 2\beta & 0
  \end{bmatrix} X(t, \beta) + \begin{bmatrix}
    u\\ v
  \end{bmatrix}$} & $X_4(\beta) \equiv (1, 0)^{\intercal}$. \\ \cline{1-1} \cline{3-3}
  5 & & $X_5(\beta) \equiv (0, 1)^{\intercal}$.  \\ \cline{1-1} \cline{3-3}
  6 & & Random in $[0, 1]\times [0, 1]$ for each $\beta \in K$. \\
  \cline{1-3}
  7& \multirow{2}{*}{$\renewcommand{\arraystretch}{1.2}\frac{\mathrm{d}}{\mathrm{d}t}X(t, \beta)  = \begin{bmatrix}
    0 & -\beta+10 \\ \beta+10 & 0
  \end{bmatrix} X(t, \beta) + \begin{bmatrix}
    u\\ v
  \end{bmatrix},$} & $X_7(\beta) \equiv (1, 0)^{\intercal}$. \\ \cline{1-1} \cline{3-3}
  8 & & $X_8(\beta) \equiv (0, 1)^{\intercal}$.  \\ 
  \cline{1-3}
  9 & $\renewcommand{\arraystretch}{1.2}\frac{\mathrm{d}}{\mathrm{d}t}X(t, \beta)  = \begin{bmatrix}
    \beta & 0\\ 0 & \beta
  \end{bmatrix} X(t, \beta) + \begin{bmatrix}
    u\\ v
  \end{bmatrix}$ & $X_9(\beta) \equiv (1, 0)^{\intercal}$.\\
  \hline
  \end{tabular}
  \caption{The system dynamics and initial conditions of the $9$ ensembles for recognition and clustering.}
      \label{tab:Ensemble clustering system dynamics}
\end{table*}

\begin{example}[Recognizing different ensembles]
  \label{ex:Recognizing.different.ensembles}
  In this part, we consider the ensemble indexed 1 - 4 in Table~\ref{tab:Ensemble clustering system dynamics}, denoted by $\Sigma_1$, \dots, $\Sigma_4$, respectively. We aim to recognize these 4 ensembles using a baseline ensemble $\Sigma_0$ indexed by $\beta \in K$, given by
  \[
    \Sigma_0 : \frac{\mathrm{d}}{\mathrm{d}t}X(t, \beta)=
    \begin{bmatrix} 0 & -\beta \\ \beta & 0 \end{bmatrix} X(t, \beta) +
    \begin{bmatrix} u \\ v \end{bmatrix},
  \]
  with the initial condition to be $X(0, \beta) \equiv (1, 0)^{\intercal}$, for all $\beta\in K$. In this experiment, we chose $K = [-10, 10]$ and $T = 1$.

  For this example, we generate a collection of $1000$ random control signals $\{u^{(k)}(t)\}_{k=1}^{1000}$, $\{v^{(k)}(t)\}_{k=1}^{1000}$, where each signal was discretized by a time-step of $0.02$ with the value at each time-step randomly drawn from $[-5, 5]$ under a uniform distribution. Then, we run $1000$ independent experiments by applying the control signals $u^{(k)}(t)$ and $v^{(k)}(t)$ on all the ensembles $\Sigma_0$, \dots, $\Sigma_4$. To collect the aggregated measurements, we randomly draw $100$ points from $K$, denoted as $\beta_i$, $i = 1, \ldots, 10$, under a uniform distribution, and computed the $1$\textsuperscript{th}-order moment as $Y^{(k)}_{j}(t) = \sum_{i=1}^{10} X^{(k)}_j(t,\beta_i)$, $j = 0, \ldots, 4$, $k = 1, \ldots , 1000$, where $X^{(k)}_j(t, \cdot)$ is the state trajectory of the $j$\textsuperscript{th} ensemble in the $k$\textsuperscript{th} experiment.

  We denote $S_j:=\{Y^{(k)}_j(t)\}_{k=1}^{1000}$, $j = 0, \ldots , 4$, as the collection of aggregated measurements from the $j$\textsuperscript{th} ensemble. In this case, the null-hypothesis of $S_i$ and $S_j$, $i, j = 0,1,2$ were generated from the same ensemble should be rejected at a significance level of $\alpha = 0.05$ if $\widehat{\mathrm{MMD}}^2(S_i, S_j) > \tfrac{4\times 1}{\sqrt{1000}}\sqrt{-\ln 0.05} = 0.219$.

  Table~\ref{tab:Ensemble recognition without calibration} displays the pairwise MMDs between $\Sigma_0$, $\Sigma_1$, \dots, $\Sigma_4$, where the constant parameter in the reproducing kernel is taken as $\sigma = 1$. It is shown in Table~\ref{tab:Ensemble recognition without calibration} that the proposed RKHS-based ensemble recognition successfully identify that $\Sigma_1 \equiv \Sigma_0$ (up to the $1$\textsuperscript{th}-order moment), while $\Sigma_2$, $\Sigma_3$, $\Sigma_4$ are not, at a significance level of $0.05$.

  \begin{table}[htb]
    \renewcommand{\arraystretch}{1.8}
    \centering
    \begin{tabular}{|C{2cm}|c|c|c|c|}
    \hline
    $j$ & $1$ & $2$ & $3$ & $4$\\
    \hline
    $\widehat{\mathrm{MMD}}^2(S_0, S_j)$ & $0.00$ & $1.59$ & $0.76 $ & $1.84$\\   
    \hline
    \end{tabular}
    \caption{The empirical MMDs between $\Sigma_0$ and $\Sigma_j$, $j = 1, \ldots ,4$ computed using $1000$ aggregate measurements.}
    \label{tab:Ensemble recognition without calibration}
  \end{table}

  If we have the prior knowledge that all the ensembles are linear, then we can calibrate the aggregated measurements to successfully identify the system dynamics of $\Sigma_0, \Sigma_1$, and $\Sigma_2$. Specifically, besides the aggregated measurements $Y_j^{(k)}(t)$, $j = 0, \ldots , 4$, $k = 1, \ldots , 1000$, we also collect the aggregated measurements of each ensemble under the null-inputs $u(t) = v(t) = 0$, denoted by $\bar{Y}_j(t)$, $j = 0\ldots , 4$. Then, we can calibrate the aggregated measurements to only identify the system dynamics, offsetting the effect of initial conditions. Since we know from linear system theory that
  \begingroup
  \allowdisplaybreaks
  \begin{align*}
    &\qquad Y_j^{(k)}(t) - \bar{Y}_j(t) = \frac{1}{100}\sum_{i=1}^{100} (X^{(k)}_j(t, \beta_i) - \bar{X}_j(t, \beta_i)) \\
    &= \frac{1}{100}\sum_{i=1}^{100} \Bigl( \Phi_j(t, 0, \beta_i)X(0, \beta_i)) \\
    &\quad +\int_0^t \Phi_j(t, s, \beta_i) B_j(s)u^{(k)}(s) \,\mathrm{d}s - \Phi_j(t, 0, \beta_i)X(0, \beta_i) \Bigr) \\
    &= \frac{1}{100}\sum_{i=1}^{100} \int_0^t \Phi_j(t, s, \beta_i) B_j(s)u^{(k)}(s) \,\mathrm{d}s,
  \end{align*}
  \endgroup
  where $X^{(k)}_j(t, \cdot )$, $j = 0, \ldots , 4$ is the state trajectory of the $j$\textsuperscript{th} ensemble under the control inputs $u^{(k)}(t)$ and $v^{(k)}(t)$; $\bar{X}_j(t, \cdot )$ is the state trajectory of the $j$\textsuperscript{th} ensemble under the null-inputs $u(t) = v(t) = 0$; $\Phi_j(t, s, \beta_i)$ and $B_j(\cdot )$ are the transition matrix and the control vector field of the $j$\textsuperscript{th} ensemble, respectively.

  Finally, we calibrate the collection of aggregated measurements as $\hat{S}_j = \{Y_j^{(k)}(t) - \bar{Y}_j(t)\}_{k = 1}^{1000}$. Table~\ref{tab:Ensemble recognition with calibration} demonstrates the MMDs between $\Sigma_0$ and $\Sigma_1$, \dots, $\Sigma_4$, computed through calibrated aggregated measurements. Following the same hypothesis test above, we conclude that $\Sigma_1$, $\Sigma_2$, and $\Sigma_3$ are dynamically equivalent to $\Sigma_0$ (up to $1$\textsuperscript{th}-order moment at a significance level of $0.05$), while $\Sigma_4$ is not.
  
  \begin{table}[htbp]
    \renewcommand{\arraystretch}{1.8}
    \centering
    \begin{tabular}{|C{2cm}|c|c|c|c|}
      \hline
      $j$ & $1$ & $2$ & $3$ & $4$\\
      \hline
      $\widehat{\mathrm{MMD}}^2(\hat{S}_0, \hat{S}_j)$ & $0.00$ & $0.01$ & $0.00 $ & $1.92$\\   
      \hline
    \end{tabular}
    \caption{The empirical MMDs between $\hat{S}_0$ and $\hat{S}_j$, computed using $1000$ calibrated aggregate measurements, provided the prior knowledge that these ensembles are linear ensembles.}
    \label{tab:Ensemble recognition with calibration}
  \end{table}
\end{example}

\begin{example}[Clustering multiple ensembles]
  \label{eq:Clustering.multiple.ensembles}
  In the last example, we aim to cluster all 9 ensembles of linear systems in Table~\ref{tab:Ensemble clustering system dynamics} using aggregated measurements.

  First let us apply the randomly generated control signals $u^{(k)}(t)$ and $v^{(k)}(t)$ ($k = 1, \ldots , 1000$) in Example~\ref{ex:Recognizing.different.ensembles} to all 9 linear ensembles. Then, we take the collections of calibrated aggregated measurements as $\hat{S}_j := \{Y_j(t)^{(k)} - \bar{Y}_j(t)\}_{k=1}^{1000}$, $j = 1, \ldots , 9$, where $Y_j(t)^{(k)}$ and $\bar{Y}_j(t)$ are the $1$\textsuperscript{th}-order moment trajectory of $\Sigma_j$ under the control inputs $[u^{(k)}(t), v^{(k)}(t)]^\intercal$ and the null-inputs $u(t) = v(t) \equiv 0$, respectively. The pairwise MMDs among all 9 ensembles are presented in Table~\ref{tab:Ensemble clustering}. Figure~\ref{fig: ensemble clustering} presents the agglomerative hierarchical cluster tree of the 9 ensembles, based on the pairwise MMDs in Table~\ref{tab:Ensemble clustering}. It is shown in Figure~\ref{fig: ensemble clustering} that the $9$ linear ensembles are clustered into $4$ clusters, namely $\{\Sigma_1, \Sigma_2, \Sigma_3\}$, $\{\Sigma_4, \Sigma_5, \Sigma_6\}$, $\{\Sigma_7, \Sigma_8\}$ and $\{\Sigma_9\}$, where the ensembles in each cluster have the same dynamics.
  
  \begin{figure}[h]
    \centering
    \includegraphics[width=0.65\linewidth]{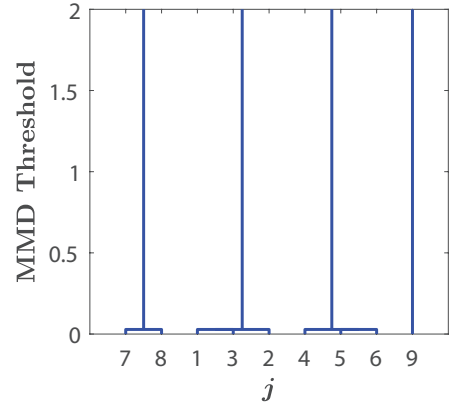}
    \caption{Agglomerative hierarchical cluster tree of the $9$ ensembles in Table~\ref{tab:Ensemble clustering system dynamics}, based on pairwise MMDs computed through calibrated moment trajectories.}
    \label{fig: ensemble clustering}
  \end{figure}
\end{example}

\begin{table*}[htbp]
  \renewcommand{\arraystretch}{1.8}
  \centering
  \begin{tabular}{|c|C{1.25cm}|C{1.25cm}|C{1.25cm}|C{1.25cm}|C{1.25cm}|C{1.25cm}|C{1.25cm}|C{1.25cm}|C{1.25cm}|}
  \hline
  \multicolumn{10}{|c|}{$\widehat{\mathrm{MMD}}^2(\hat{S}_i, \hat{S}_j)$}\\
  \hline
  \diagbox[height = 0.7cm]{$i$}{$j$}  & $1$ & $2$ & $3$ & $4$ & $5$ & $6$ & $7$ & $8$ & $9$\\
  \hline
  $1$ & $0.02$ & $0.02$  & $0.02$ & $1.71$ & $1.71$ & $1.71$ & $1.71$ & $1.71$ & $0.86$\\
  \hline
  $2$ & $0.02$ & $0.02$  & $0.02$ & $1.71$ & $1.71$ & $1.71$ & $1.71$ & $1.71$ & $0.86$\\
  \hline
  $3$ & $0.02$ & $0.02$  & $0.02$ & $1.71$ & $1.71$ & $1.71$ & $1.71$ & $1.71$ & $0.86$\\
  \hline
  $4$ & $1.71$ & $1.71$ & $1.71$ & $0.02$ & $0.02$  & $0.02$ & $1.71$ & $1.71$ & $0.85$\\
  \hline 
  $5$ & $1.71$ & $1.71$ & $1.71$ & $0.02$ & $0.02$  & $0.02$ & $1.71$ & $1.71$ & $0.85$\\
  \hline 
  $6$ & $1.71$ & $1.71$ & $1.71$ & $0.02$ & $0.02$  & $0.02$ & $1.71$ & $1.71$ & $0.85$\\
  \hline 
  $7$ & $1.71$ & $1.71$ & $1.71$ & $1.71$ & $1.71$ & $1.71$ & $0.02$ & $0.02$ & $0.86$\\
  \hline 
  $8$ & $1.71$ & $1.71$ & $1.71$ & $1.71$ & $1.71$ & $1.71$ & $0.02$ & $0.02$ & $0.86$\\
  \hline
  $9$ & $0.86$ & $0.86$ &$0.86$ &$0.85$ &$0.85$ &$0.85$ &$0.86$ &$0.86$ & $0.02$\\
  \hline
  \end{tabular}
  \caption{The pairwise MMDs among the unknown ensembles $\Sigma_j$'s computed using the calibrated aggregated measurements $\hat{S}_j$, $j = 1, \ldots ,9$, provided the prior knowledge that these ensembles are all linear ensembles. }
  \label{tab:Ensemble clustering}
\end{table*}

\section{Conclusions}
\label{sec: conclusion}

  In this paper, we propose a novel framework to learn the system dynamics of an unknown ensemble using its aggregated measurements in reproducing kernel Hilbert spaces. We demonstrate that we can identify unknown ensembles by comparing the distribution of their aggregated measurements after perturbing the systems using random control inputs. Additionally, by introducing a new notion of aggregated Markov parameters, we also provide a systematic approach to recognize and cluster unknown ensembles by linear approximation, without the need of a baseline model or any prior knowledge on the system dynamics. Ample examples are included in the paper to illustrate the efficacy and robustness of our framework.

\section*{Appendix}

In this appendix, we review some technical results regarding RKHS and universal kernels that are necessary to the development of our framework.

\begin{definition}
  \label{def: universal kernel}
  Let $\mathcal{X}$ be a compact metric space. A continuous kernel $\rho:\mathcal{X}\times \mathcal{X}\to \mathbb{R}$ is said to be universal if the RKHS $\mathcal{H}$ associated with $\rho$ is dense in $C(X)$ with respect to sup-norm, i.e., for every $f\in C(\mathcal{X})$ and $\epsilon>0$, there exists $g\in \mathcal{H}$ such that $\|f - g\|_{L^\infty } < \epsilon$.
\end{definition}

\begin{theorem}
  \label{thm: universality.GRBF}
  Let $\mathcal{X}$ be a compact metric space, $\mathcal{H}$ be a separable Hilbert space, and $\iota: \mathcal{X}\to \mathcal{H}$ be an injective map. Then, the Gaussian-RBF-type kernel $\rho:\mathcal{X}\times \mathcal{X}\to \mathbb{R}$, given by
  \begin{equation}
    \label{equ: Gaussian.RBF.kernel}
    \rho(x_1, x_2) := \exp (-\sigma \|\iota(x_1) - \iota(x_2)\|_{\mathcal{H}}^2), 
  \end{equation}
  is a universal kernel, where $\sigma >0$ and $\|\cdot \|_{\mathcal{H}}$ denotes the norm induced by inner product of $\mathcal{H}$.
\end{theorem}
\begin{proof}
    See \cite{NIPS2010_Universal_kernel}, Theorem 2.2.
\end{proof}

\begin{lemma}\label{lem: existence of kernel mean embedding}
    Let $\mathcal{X}$ be a non-empty set, $\rho:\mathcal{X}\times \mathcal{X} \to \mathbb{R}$ be a reproducing kernel, and $\mathcal{H}$ be its associated RKHS. Given a probability distribution defined on $\mathbb{P}$, if $\rho( \cdot , \cdot )$ is measurable and $\mathbb{E}_{x}[\sqrt{\rho(x, x)}]:= \int_{\mathcal{X}} \sqrt{\rho(x, x)} \,\mathrm{d}\mathbb{P}(x) < \infty$, then the kernel mean embedding of $\mathbb{P}$, denoted as $\mu_{\mathbb{P}}$, is in $\mathcal{H}$, i.e.,
    \[
      \mu_{\mathbb{P}} = \int_{\mathcal{X}} \rho( \cdot , x) \,\mathrm{d}\mathbb{P}(x)\in \mathcal{H}.
    \]
\end{lemma}

\begin{proof}
  See Lemma~3 in \cite{gretton12a}.
\end{proof}

\begin{lemma}
  \label{lem: AMP.in.l2}
  Let $K\subset \mathbb{R}$ be compact, $A(\cdot )\in C(K, \mathbb{R}^{n\times n})$, $B(\cdot )\in C(K, \mathbb{R}^{n\times 1})$, and $C(\cdot )\in C(K, \mathbb{R}^{1 \times n})$. The infinite sequence $\eta = \{\eta_j\}_{j=1}^{\infty }$, where
  \[
    \eta_j = \frac{1}{j!}\int_K C(\beta) A^j(\beta)B(\beta) \,\mathrm{d}\beta
  \]
  is square-summable, i.e., $\sum_{j=0}^{\infty } \eta^2_j < \infty$, $\eta\in\ell^2$. Furthermore, it holds that
  \[
    \lim_{j\to \infty } (j!\eta_j)^{\frac{1}{j}} < \infty.
  \]
\end{lemma}

\begin{proof}
    It suffices to show that $\sum_{j=0}^{\infty } |\eta_j| < \infty$. This is evident since
    \[
      \sum_{j=0}^{\infty } |\eta_j|\leq \int_{K} |C(\beta)|e^{|A(\beta)|} |B(\beta)| \,\mathrm{d} \beta < \infty,
    \]
    where $|A(\beta)|$, $|B(\beta)|$, and $|C(\beta)|$ are the element-wise absolute values of $A(\beta)$, $B(\beta)$, and $C(\beta)$, respectively; and the second inequality holds since $|C(\cdot )|e^{|A(\cdot )|} |B(\cdot )|$ is continuous.

    To show that $\lim_{j\to \infty } (j!\eta_j)^{\frac{1}{j}} < \infty $, it suffices to observe that
    \begin{align*}
      \lim_{j\to \infty } & (j!\eta_j)^{\frac{1}{j}}  = \lim_{j\to \infty } \left( \int_K C(\beta)A^j(\beta)B(\beta) \mathrm{d}\beta \right)^{\frac{1}{j}}\\
      & \leq \lim_{j\to \infty } (|K| \|C\| \|A\|^j \|B\| )^{\frac{1}{j}} = \|A\| < \infty,
    \end{align*}
    where $\|\cdot \|$ denotes the sup-norm over $K$ and $|K|:= \int_K 1 \,\mathrm{d} \beta$ is the Lebesgue measure of $K$.
\end{proof}

\begin{lemma}
  \label{lem: MGF expansion}
  Let $X$ be a random variable. If the limit $\lim_{n\to \infty} (\mathbb{E}[X^n])^{\frac{1}{n} }$ exists, then the moment generating function of $X$, denoted as $M_X$, has a power series expansion over the whole real line, i.e.
  \[
    M_X(s) := \mathbb{E}[e^{sX}] = \sum_{k=0}^{\infty} \frac{s^k}{k!} \mathbb{E}[X^k],
  \]
  for all $s\in \mathbb{R}$.
\end{lemma}

\begin{proof}
  It is not hard to observe that the power series $\sum_{k=0}^{\infty} \tfrac{s^k}{k!} \mathbb{E}[X^k]$ has a radius of convergence, denoted as $R$, such that
  \begin{equation}
    \label{equ: lem MGF-1}
    \frac{1}{R} = \mathrm{limsup}_{n\to\infty} \biggl( \frac{1}{n!} \mathbb{E}\:[X^n]  \biggr)^{\frac{1}{n}}.
  \end{equation}
  Let $C = \lim_{n\to \infty} (\mathbb{E}[X^n])^{\frac{1}{n} }$. Then, \eqref{equ: lem MGF-1} is computed as
  \begin{equation}
      R = \frac{1}{C}\lim_{n\to\infty} (n!)^{\frac{1}{n}}.
  \end{equation}
  By Stirling's formula, $n! \approx \sqrt{2\pi n} n^n e^{-n}$ when $n$ is sufficiently large. Hence $(n!)^{\frac{1}{n}}\to \infty$ as $n\to \infty$, which implies that $R=\infty$. Therefore, the power series $\sum_{k=0}^{\infty} \tfrac{s^k}{k!} \mathbb{E}[X^k]$ converges for all $s \in \mathbb{R}$.
\end{proof}

  

\bibliographystyle{plain}
\bibliography{references.bib}

\begin{thebibliography}{10}

\bibitem{becker2012approximate}
Aaron Becker and Timothy Bretl.
\newblock Approximate steering of a unicycle under bounded model perturbation
  using ensemble control.
\newblock {\em IEEE Transactions on Robotics}, 28(3):580--591, 2012.

\bibitem{brown2004phase}
Eric Brown, Jeff Moehlis, and Philip Holmes.
\newblock On the phase reduction and response dynamics of neural oscillator
  populations.
\newblock {\em Neural Computation}, 16(4):673--715, 2004.

\bibitem{chen2019structure}
Xudong Chen.
\newblock Structure theory for ensemble controllability, observability, and
  duality.
\newblock {\em Mathematics of Control, Signals, and Systems}, 31:1--40, 2019.

\bibitem{NIPS2010_Universal_kernel}
Andreas Christmann and Ingo Steinwart.
\newblock Universal kernels on non-standard input spaces.
\newblock In J.~Lafferty, C.~Williams, J.~Shawe-Taylor, R.~Zemel, and
  A.~Culotta, editors, {\em Advances in Neural Information Processing Systems},
  volume~23, pages 406--414. Curran Associates, Inc., 2010.

\bibitem{Cortes08}
Corinna Cortes, Mehryar Mohri, Michael Riley, and Afshin Rostamizadeh.
\newblock Sample selection bias correction theory.
\newblock In Yoav Freund, L{\'a}szl{\'o} Gy{\"o}rfi, Gy{\"o}rgy Tur{\'a}n, and
  Thomas Zeugmann, editors, {\em Algorithmic Learning Theory}, pages 38--53.
  Springer Berlin Heidelberg, 2008.

\bibitem{dirr2018uniform}
Gunther Dirr and Michael Schönlein.
\newblock Uniform and {$L^q$}-ensemble reachability of parameter-dependent
  linear systems, 2018.
\newblock \href{https://arxiv.org/abs/1810.09117}{arXiv:1810.09117 [math.OC]}.

\bibitem{glaser1998unitary}
Steffen~J Glaser, T~Schulte-Herbr{\"u}ggen, M~Sieveking, O~Schedletzky,
  Niels~Christian Nielsen, Ole~Winneche S{\o}rensen, and C~Griesinger.
\newblock Unitary control in quantum ensembles: Maximizing signal intensity in
  coherent spectroscopy.
\newblock {\em Science}, 280(5362):421--424, 1998.

\bibitem{gretton12a}
Arthur Gretton, Karsten~M. Borgwardt, Malte~J. Rasch, Bernhard Sch{{\"o}}lkopf,
  and Alexander Smola.
\newblock A kernel two-sample test.
\newblock {\em Journal of Machine Learning Research}, 13(25):723--773, 2012.

\bibitem{helmke2014uniform}
Uwe Helmke and Michael Schönlein.
\newblock Uniform ensemble controllability for one-parameter families of
  time-invariant linear systems.
\newblock {\em Systems \& Control Letters}, 71:69 -- 77, 2014.

\bibitem{hoeffding1994probability}
Wassily Hoeffding.
\newblock Probability inequalities for sums of bounded random variables.
\newblock In {\em The Collected Works of Wassily Hoeffding}, pages 409--426.
  Springer, 1994.

\bibitem{kafashan2015optimal}
MohammadMehdi Kafashan and ShiNung Ching.
\newblock Optimal stimulus scheduling for active estimation of evoked brain
  networks.
\newblock {\em Journal of Neural Engineering}, 12(6):066011, Oct 2015.

\bibitem{kimeldorf1971}
George Kimeldorf and Grace Wahba.
\newblock Some results on tchebycheffian spline functions.
\newblock {\em Journal of Mathematical Analysis and Applications},
  33(1):82--95, 1971.

\bibitem{kuritz2018ensemble}
Karsten Kuritz, Shen Zeng, and Frank Allg{\"o}wer.
\newblock Ensemble controllability of cellular oscillators.
\newblock {\em IEEE Control Systems Letters}, 3(2):296--301, 2018.

\bibitem{li2011ensemble}
Jr-Shin Li.
\newblock Ensemble control of finite-dimensional time-varying linear systems.
\newblock {\em IEEE Transactions on Automatic Control}, 56(2):345--357, 2011.

\bibitem{li2013control}
Jr-Shin Li, Isuru Dasanayake, and Justin Ruths.
\newblock Control and synchronization of neuron ensembles.
\newblock {\em IEEE Transactions on Automatic Control}, 58(8):1919--1930, 2013.

\bibitem{li2011optimal}
Jr-Shin Li, Justin Ruths, Tsyr-Yan Yu, Haribabu Arthanari, and Gerhard Wagner.
\newblock Optimal pulse design in quantum control: A unified computational
  method.
\newblock {\em Proceedings of the National Academy of Sciences}, 2011.

\bibitem{li2019separating}
Jr-Shin Li, Wei Zhang, and Lin Tie.
\newblock On separating points for ensemble controllability.
\newblock {\em SIAM Journal on Control and Optimization}, 58(5):2740--2764,
  2020.

\bibitem{miao2020numerical}
Wei Miao, Gong Cheng, and Jr-Shin Li.
\newblock On numerical examination of uniform ensemble controllability for
  linear ensemble systems.
\newblock {\em IEEE Control Systems Letters}, 5(6):1898--1903, 2020.

\bibitem{miao2020convexgeometric}
Wei Miao and Jr-Shin Li.
\newblock A convex-geometric approach to ensemble control analysis and design
  in a hilbert space, 2020.
\newblock \href{https://arxiv.org/abs/2003.09987}{arXiv:2003.09987 [math.OC]}.

\bibitem{miao2020geometric}
Wei Miao and Jr-Shin Li.
\newblock A geometric approach to linear ensemble control analysis and design.
\newblock In {\em 2020 American Control Conference (ACC)}, pages 4600--4605.
  IEEE, 2020.

\bibitem{rosenblum2004controlling}
Michael~G Rosenblum and Arkady~S Pikovsky.
\newblock Controlling synchronization in an ensemble of globally coupled
  oscillators.
\newblock {\em Physical Review Letters}, 92(11):114102, 2004.

\bibitem{schonlein2016controllability}
Michael Schönlein and Uwe Helmke.
\newblock Controllability of ensembles of linear dynamical systems.
\newblock {\em Mathematics and Computers in Simulation}, 125:3 -- 14, 2016.

\bibitem{tabuada2020universal}
Paulo Tabuada and Bahman Gharesifard.
\newblock Universal approximation power of deep neural networks via nonlinear
  control theory.
\newblock {\em arXiv preprint arXiv:2007.06007}, 2020.

\bibitem{tie2017explicit}
Lin Tie, Wei Zhang, Shen Zeng, and Jr-Shin Li.
\newblock Explicit input signal design for stable linear ensemble systems.
\newblock {\em IFAC-PapersOnLine}, 50(1):3051--3056, 2017.
\newblock 20th IFAC World Congress.

\bibitem{Li_SICON17}
Shuo Wang and Jr-Shin Li.
\newblock Fixed-endpoint optimal control of bilinear ensemble systems.
\newblock {\em SIAM Journal on Control and Optimization}, 55(5):3039--3065,
  2017.

\bibitem{Li_Automatica18}
Shuo Wang and Jr-Shin Li.
\newblock Free-endpoint optimal control of inhomogeneous bilinear ensemble
  systems.
\newblock {\em Automatica}, 95:306 -- 315, 2018.

\bibitem{weinan2017proposal}
E~Weinan.
\newblock A proposal on machine learning via dynamical systems.
\newblock {\em Communications in Mathematics and Statistics}, 5(1):1--11, 2017.

\bibitem{zeng2016moment}
Shen Zeng and Frank Allgoewer.
\newblock A moment-based approach to ensemble controllability of linear
  systems.
\newblock {\em Systems \& Control Letters}, 98:49--56, 2016.

\bibitem{zeng2017sampled}
Shen Zeng, Hideaki Ishii, and Frank Allg{\"o}wer.
\newblock Sampled observability and state estimation of linear discrete
  ensembles.
\newblock {\em IEEE Transactions on Automatic Control}, 62(5):2406--2418, 2017.

\bibitem{zeng2018computation}
Shen Zeng, Wei Zhang, and Jr-Shin Li.
\newblock On the computation of control inputs for linear ensembles.
\newblock In {\em 2018 Annual American Control Conference (ACC)}, pages
  6101--6107. IEEE, 2018.

\bibitem{Li_ACC12_SVD}
Anatoly Zlotnik and Jr-Shin Li.
\newblock Synthesis of optimal ensemble controls for linear systems using the
  singular value decomposition.
\newblock In {\em 2012 American Control Conference (ACC)}, pages 5849--5854,
  June 2012.

\bibitem{Li_NatureComm16}
Anatoly Zlotnik, Raphael Nagao, Istv{\'a}n~Z Kiss, and Jr-Shin Li.
\newblock Phase-selective entrainment of nonlinear oscillator ensembles.
\newblock {\em Nature Communications}, 7(1):1--7, 2016.

\end{thebibliography}

\end{document}